\documentclass{article}
\usepackage{graphics,amsthm,fullpage}

\usepackage{mathrsfs}
\usepackage{amssymb,amsmath}
\usepackage{perpage} 
\usepackage{array}
\usepackage[usenames,dvipsnames]{color}
\usepackage[pdfauthor={Palash Dey},pdfnewwindow=true,colorlinks=true,linkcolor=blue,citecolor=Mahogany,pagebackref=true,breaklinks=true]{hyperref}

\usepackage{enumitem}

\usepackage{multirow}

\usepackage{thmtools}

\MakePerPage{footnote} 

\usepackage{mathrsfs}
\usepackage{amssymb,amsmath}
\usepackage{array}

\usepackage{enumitem}

\usepackage{graphicx}
\usepackage{xspace}

\usepackage{amsopn}
\usepackage{mathtools}

\usepackage{subfig}
\usepackage[hang,flushmargin]{footmisc} 

\usepackage{enumitem}

\newcommand{\NP}{\ensuremath{\mathsf{NP}}\xspace}
\newcommand{\coNP}{\ensuremath{\mathsf{co}}-\ensuremath{\mathsf{NP}}\xspace}
\newcommand{\NPH}{\ensuremath{\mathsf{NPH}}\xspace}
\newcommand{\NPC}{\ensuremath{\mathsf{NPC}}\xspace}

\newcommand{\Pb}{\ensuremath{\mathsf{P}}\xspace}

\usepackage{xspace}

\newtheorem{proposition}{\bf Proposition}

\newtheorem{theorem}{\bf Theorem}
\newtheorem{lemma}{\bf Lemma}

\newtheorem{definition}{\bf Definition}

\begin{document}
\sloppy

\title{Manipulation is Harder with Incomplete Votes}
\author{Palash Dey, Neeldhara Misra, and Y. Narahari\\ \texttt{palash@csa.iisc.ernet.in, mail@neeldhara.com, hari@csa.iisc.ernet.in}\\
Department of Computer Science and Automation \\Indian Institute of Science - Bangalore, India.}

\maketitle

\begin{abstract}
The Coalitional Manipulation (CM) problem has been
studied extensively in the literature for many voting rules. 
The CM problem, however, has
been studied only in the complete information setting, that is, when the
manipulators know the votes of the non-manipulators. A more realistic
scenario is an incomplete information setting where the manipulators do
not know the exact votes of the non-manipulators but may have some partial
knowledge of the votes. In this paper, we study a setting where the
manipulators know a partial order for each voter that is consistent with
the vote of that voter. 

In this setting, we introduce and study two natural
computational problems -
(1) Weak Manipulation (WM) problem where the manipulators wish to
vote in a way that makes their preferred candidate win in at least one
extension of the partial votes of the non-manipulators;
(2) Strong Manipulation (SM) problem where the manipulators wish to vote
in a way that makes their preferred candidate win in all possible
extensions of the partial votes of the non-manipulators. 
We study the computational complexity of the WM and the SM problems for commonly used voting rules
such as plurality, veto, $k$-approval, $k$-veto, maximin,
Copeland, and Bucklin. Our key finding is that, barring a few exceptions,
manipulation becomes a significantly harder problem in the setting of
incomplete votes.
\end{abstract}

\section{Introduction}

In many real life and AI related applications, agents often need to agree upon a common decision although they have different preferences over the available candidates. A natural tool used in these situations is voting. Some classic examples of the use of voting rules in the context of multiagent systems include Clarke tax~\cite{ephrati1991clarke} and collaborative filtering~\cite{pennock2000social} etc.

In a typical voting scenario, we have a set of candidates and a set of voters 
reporting their rankings of the candidates called their preferences or votes. A voting rule 
selects one candidate as the winner once all voters provide their votes. A set of votes over a set of 
candidates along with a voting rule is called an election. 
A central issue in voting is the possibility of \emph{manipulation}. For many voting rules, it turns out that even a single vote, if cast differently, can alter the outcome. In particular, a voter manipulates the voting rule if, by misrepresenting her preference, she obtains an outcome that she prefers over the ``honest'' outcome. By the celebrated theorem of Gibbard and Satterthwaite~\cite{gibbard1973manipulation,satterthwaite1975strategy}, every unanimous and non-dictatorial voting rule with three candidates or more is manipulable. We refer to~\cite{brandt2012computational} for an excellent survey material on computational social choice theory.

\subsection{Background}

Considering that voting rules are indeed susceptible to manipulation, it is natural to seek ways by which elections can be protected from manipulations. The works of Bartholdi et al.~\cite{bartholdi1989computational,bartholdi1991single} approach the problem from the perspective of computational intractability. They exploit the possibility that voting rules, despite being vulnerable to manipulation in theory, may be hard to manipulate in practice. Indeed, a manipulator is faced with the following decision problem: given a voting rule $r$, a distinguished candidate $c$, and a collection of votes $\mathcal{P}$, does there exist a vote $v$ that, when tallied with $\mathcal{P}$, makes $c$ win the election? The manipulation problem has subsequently been generalized to the problem of \textsc{Coalitional manipulation (CM)} by Conitzer et al.~\cite{conitzer2007elections}, where one or more manipulators collude together and try to make a distinguished candidate win the election. The manipulation problem, fortunately, turns out to be \NP{}-Hard (\NPH{}) in several settings. This established the success of the approach of demonstrating a computational barrier to manipulation.

However, despite having set out to demonstrate the hardness of manipulation, the initial results in~\cite{bartholdi1989computational} were to the contrary, indicating that many voting rules are in fact easy to manipulate. Moreover, even with multiple manipulators involved, popular voting rules like plurality, veto, $k$-approval, and Bucklin continue to be easy to manipulate~\cite{xia2009complexity}. While the computational barrier approach works for rules for which the coalitional manipulation problem turns out to be \NPH{}, in all other cases the possibility of manipulation remains a very legitimate concern. 

\subsection{Motivation}

In our work, we propose to extend the argument of computational intractability to address the cases where the approach appears to fail. We note that most incarnations of the manipulation problem studied so far are in the complete information setting, where the manipulators have complete knowledge of the preferences of the truthful voters. While these assumptions are indeed the best possible for the computationally negative results, we note that they are not necessarily reflective of realistic settings. Indeed, concerns regarding privacy of information, and in other cases, the sheer volume of information, would be necessary hurdles for manipulators to obtain complete information. Motivated by this, we consider the manipulation problem in a natural \emph{partial information} setting. In particular, we model the partial information of the manipulators about the votes of the non-manipulators as partial orders over the set of candidates. A partial order over the set of candidates will be called a partial vote. Our results show that several of the voting rules that are easy to manipulate in the complete information setting become \NPH{} when the manipulators know only partial votes. Indeed, for many voting rules, we show that even if a small number of preferences are missing from the profile, manipulation becomes an~\NPH{} problem. 

We first introduce two new computational problems that, in a natural way, extend the problem of manipulation to the partial information setting. In both the problems, the input is a set of partial votes $\mathcal{P}$ corresponding to the votes of the non-manipulators, a non-empty set of manipulators $M$, and a preferred candidate $c$. The task in the \textsc{Weak Manipulation (WM)} problem is to determine if there is a way of casting the manipulators' votes such that $c$ wins the election for at least one extension of the partial votes in $\mathcal{P}$. On the other hand, in the \textsc{Strong Manipulation (SM)} problem, we would like to know if there is a way of casting the manipulators' votes such that $c$ wins the election in \emph{every extensions} of the partial votes in $\mathcal{P}$. Both of these problems generalize CM, and hence any computational intractability result for CM immediately yields a corresponding intractability result for WM and SM under the same setting. For example, it is known that the CM problem is \NPH{} for the maximin voting rule when we have at least two manipulators~\cite{xia2009complexity}. This implies that, both the WM and the SM problems are \NPH{} for the maximin voting rule when we have at least two manipulators. 

\subsection{Related Work and Contributions}

Conitzer et al.~\cite{conitzer2011dominating} also study manipulation under partial information setting. However, they focused on whether or not there exists a \emph{dominating manipulation} and shows that it is \NPH{} for many common voting rules. Given some partial votes, a dominating manipulation is a non-truthful vote that the manipulator can cast which makes the winner at least as preferable (and sometimes more preferable) as the winner when the manipulator votes truthfully. The dominating manipulation problem and the WM and the SM problems do not seem to have any apparent complexity theoretic connection. For example, dominating manipulation problem is \NPH{} for all the common voting rules except plurality and veto, whereas, the SM problem is easy for most of the cases (see \MakeUppercase table\nobreakspace \ref {table:partial_summary}). Elkind et al.~\cite{elkind2012manipulation} studies manipulation under voting rule uncertainty. However, in our work, the voting rule is fixed and known to the manipulators.

Two closely related problems that have been extensively studied in the context of incomplete votes are \textsc{Possible Winner (PW)} and \textsc{Necessary Winner (NW)}~\cite{konczak2005voting}. In the PW problem, we are given a set of partial votes $\mathcal{P}$ and a candidate $c$, and the question is whether there exists an extension of $\mathcal{P}$ where $c$ wins, while in the NW problem, the question is whether $c$ is a winner in all extensions of $\mathcal{P}$. Following the work of Konczak et al., a number of  special cases and variants of the PW problem have been studied in the literature~\cite{chevaleyre2010possible,bachrach2010probabilistic,baumeister2011computational,baumeister2012possible,gaspers2014possible,xia2011determining,ding2013voting}. Clearly, the WM problem is a specialization of the PW problem where there exist at least one empty vote (see \MakeUppercase proposition\nobreakspace \ref {pw_hard}).  On the other hand, the SM problem is significantly different from the NW problem --- the NW problem is known to belong to the complexity class \coNP{}, whereas the SM problem does not belong to \coNP{} unless \NP{}$=$\coNP{} (since, the SM problem is \NPH{} for some voting rules, for example the Borda voting rule). However, the SM problem clearly belongs to the complexity class $\Sigma_p^2$.

Here, we undertake a study of computational complexity of the WM and the SM problems for plurality, veto, $k$-approval, $k$-veto, maximin, Copeland, and Bucklin voting rules. Note that, for many other common voting rules like single transferable vote, ranked pairs etc., the existing hardness results for the CM problem immediately imply hardness of the WM and the SM problems. Among the voting rules that we have explored, only the plurality and veto voting rules continue to be vulnerable to manipulation in all versions of the partial information setting. We show polynomial time algorithms for both the problems for these voting rules, for any coalition size. 

With a few exceptions, therefore, we find that manipulation becomes a significantly harder problem when not all votes are available for the perusal of the manipulators. We summarize the results of this paper in \MakeUppercase table\nobreakspace \ref {table:partial_summary}.

\begin{table}[!htbp]
  \centering
  {\renewcommand{\arraystretch}{1.7}
 \begin{tabular}{|c|c|c|c|c|}\hline
  Voting Rule			& WM, $c=1$	&WM	&SM, $c=1$	&SM	\\\hline\hline
  Plurality 	& {\Pb{}}	~ \MakeUppercase theorem\nobreakspace \ref {sm_maximin_easy}	&\Pb{} ~ \MakeUppercase theorem\nobreakspace \ref {sm_maximin_easy}	&\Pb{} ~ \MakeUppercase theorem\nobreakspace \ref {sm_k_easy}		&\Pb{} ~ \MakeUppercase theorem\nobreakspace \ref {sm_k_easy}	\\\hline
  Veto 		& \Pb{}	~ \MakeUppercase theorem\nobreakspace \ref {sm_maximin_easy}	&\Pb{} ~ \MakeUppercase theorem\nobreakspace \ref {sm_maximin_easy}	&\Pb{} ~ \MakeUppercase theorem\nobreakspace \ref {sm_k_easy}		&\Pb{} ~ \MakeUppercase theorem\nobreakspace \ref {sm_k_easy}	\\\hline
  $k$-approval	& {\bf\NPC{}} ~ \MakeUppercase theorem\nobreakspace \ref {wm_k_hard}	&{\bf\NPC{}} ~ \MakeUppercase theorem\nobreakspace \ref {wm_k_hard}	&\Pb{} ~ \MakeUppercase theorem\nobreakspace \ref {sm_k_easy}		&\Pb{} ~ \MakeUppercase theorem\nobreakspace \ref {sm_k_easy}	\\\hline
  $k$-veto	& {\bf\NPC{}} ~ \MakeUppercase theorem\nobreakspace \ref {k_veto_wm}	&{\bf\NPC{}} ~ \MakeUppercase theorem\nobreakspace \ref {k_veto_wm}	&\Pb{} ~ \MakeUppercase theorem\nobreakspace \ref {sm_k_easy}		&\Pb{} ~ \MakeUppercase theorem\nobreakspace \ref {sm_k_easy}	\\\hline
  Borda	& ?	&\NPC{}	&\Pb{} ~ \MakeUppercase theorem\nobreakspace \ref {sm_maximin_easy}		&\NPH{}		\\\hline
  Maximin			& {\bf\NPC{}} ~ \MakeUppercase theorem\nobreakspace \ref {wm_maximin_hard}	&\NPC{} &\Pb{} ~ \MakeUppercase theorem\nobreakspace \ref {sm_maximin_easy}	&\NPH{}		\\\hline
  Copeland			& {\bf\NPC{}} ~ \MakeUppercase theorem\nobreakspace \ref {wm_copeland_hard}&\NPC{}	&{\bf\NPH{}} ~ \MakeUppercase theorem\nobreakspace \ref {wm_copeland_hard}	&\NPH{}		\\\hline
  Bucklin			& {\bf\NPC{}} ~ \MakeUppercase theorem\nobreakspace \ref {wm_copeland_hard}	&{\bf\NPC{}} ~ \MakeUppercase theorem\nobreakspace \ref {wm_copeland_hard}	&\Pb{} ~ \MakeUppercase theorem\nobreakspace \ref {sm_bucklin_easy}	&\Pb{} ~ \MakeUppercase theorem\nobreakspace \ref {sm_bucklin_easy}	\\\hline
 \end{tabular}}
 \caption{\normalfont Summary of results ($c$ denotes coalition size). Only those results that are proved in this paper contain reference to the Theorem proving it. For $k$-approval and $k$-veto, the results hold when $k = O(1)$ and $k>1$. The '?' mark means that the problem is open.}
 \label{table:partial_summary}
\end{table}

The WM problem for the Borda voting rule when we have only one manipulator is open. Note that, we know that the CM problem for the Borda voting rule is \NP{}-complete (\NPC{})~\cite{davies2011complexity,betzler2011unweighted}, when we have at least two manipulators. However, this does not shed any light on the complexity of the WM problem for the Borda voting rule, when we have only one manipulator. 

\textbf{Organization.}~The rest of the paper is organized as follows. We introduce the basic terminologies in Section~\ref{sec:prelim}, formally define the problems in Section~\ref{sec:probdef}, show hardness results in Section~\ref{sec:hard}, present polynomial time algorithms in Section~\ref{sec:poly} and finally conclude in Section~\ref{sec:con}.


\section{Preliminaries}\label{sec:prelim}

Let $\mathcal{V}=\{v_1, \dots, v_n\}$ be the set of all \emph{voters} and $\mathcal{C}=\{c_1, \dots, c_m\}$ the set of all \emph{candidates}. 
If not specified explicitly, $n$ always denotes total number of voters and $m$ denotes total number of candidates. 
Each voter $v_i$'s \textit{vote} is a \emph{preference} $\succ_i$ over the 
candidates which is a linear order over $\mathcal{C}$. 
For example, for two candidates $a$ and $b$, $a \succ_i b$ means that the voter $v_i$ prefers $a$ to $b$. We denote the set of all linear orders over $\mathcal{C}$ by $\mathcal{L(C)}$. Hence, $\mathcal{L(C)}^n$ denotes the set of all $n$-voters' preference profile $(\succ_1, \dots, \succ_n)$. 
A map $r:\cup_{n,|\mathcal{C}|\in\mathbb{N}^+}\mathcal{L(C)}^n\longrightarrow 2^\mathcal{C}\setminus\{\emptyset\}$
is called a \emph{voting rule}. For some preference profile $(\succ_1, \dots, \succ_n)\in \mathcal{L(C)}^n$, if $\{w\}=r(\succ_1, \dots, \succ_n)$, then we say $w$ wins uniquely. From here on, whenever we say some candidate $w$ wins, we mean that the candidate $w$ is winning uniquely. For simplicity, we restrict ourselves to the unique winner case in this paper. All our proofs can be easily extended for the co-winner case. 
A more general setting is an {\em election\/} where the votes are only 
\emph{partial orders} over candidates. A \emph{partial order} is a relation that is \emph{reflexive, 
antisymmetric}, and \emph{transitive}. A partial vote can be extended to possibly more than one linear votes depending on how 
we fix the order for the unspecified pairs of candidates. For example, in an election with the set of candidates $\mathcal{C} = \{a, b, c\}$, 
a valid partial vote can be $a \succ b$. This partial vote can be extended to three linear votes namely, $a \succ b \succ c$, $a \succ c \succ b$, 
$c \succ a \succ b$. Whenever we do not specify ordering among some number of candidates while describing a partial vote, the statement or proof is valid irrespective of the order we fix the unspecified ordering. Next we give examples of some common voting rules below.
\begin{itemize}
 \item Positional scoring rules : A collection of $m$-dimensional vectors $\overrightarrow{s_m}=\left(\alpha_1,\alpha_2,\dots,\alpha_m\right)\in\mathbb{R}^m$ with $\alpha_1\ge\alpha_2\ge\dots\ge\alpha_m$ and $\alpha_1>\alpha_m$ for every $m\in \mathbb{N}$ naturally defines a
 voting rule - a candidate gets score $\alpha_i$ from a vote if it is placed at the $i^{th}$ position, and the 
 score of a candidate is the sum of the scores it receives from all the votes. 
 The winners are the candidates with maximum score. Scoring rules remain unchanged if we multiply every $\alpha_i$ by any constant $\lambda>0$ and/or add any constant $\mu$. Hence, we assume without loss of generality that for any score vector $\overrightarrow{s_m}$, there exists a $j$ such that $\alpha_j - \alpha_{j+1}=1$ and $\alpha_k = 0$ for all $k>j$. We call such a $\overrightarrow{s_m}$ a normalized score vector. 
 If $\alpha_i$ is $1$ for $i\in [k]$ and $0$ otherwise, then, we get the {\bf$k$-approval} voting rule. For the {\bf$k$-veto} voting rule, $\alpha_i$ is $0$ for $i\in [m-k]$ and $-1$ otherwise. $1$-approval is called the plurality voting rule and $1$-veto is called the veto voting rule.
 \item Bucklin : A candidate $x$'s Bucklin score is the minimum number $l$ such that at least half 
 of the voters rank $x$ in their top $l$ positions. The winners are the candidates with lowest Bucklin score.
 \item Maximin : For any two candidates $x$ and $y$, $D(x,y) = N(x,y) - N(y,x)$, where $N(x,y)$ $(\text{respectively }N(y,x))$ is the number of voters who prefer $x$ to $y$ (respectively $y$ to $x$). The election we get by restricting all the votes to $x$ and $y$ only is called the pairwise election between $x$ and $y$. We call $D(x,y)$ to be the \textit{margin of victory} of $x$ over $y$ in their pairwise election. The maximin score of a candidate $x$ is $min_{y\ne x} D(x,y)$. The winners are the candidates with maximum maximin score.
 \item Copeland : The Copeland score of a candidate $x$ is the number of candidates $y\ne x$ such that $D(x,y)>0.$ The winners are the candidates with maximum Copeland score.
\end{itemize}

\section{Problem Formulation}\label{sec:probdef}

We first recall the definitions of the \textsc{Coalitional Manipulation (CM)} and \textsc{Possible Winner (PW)} problems below for ease of access.

\begin{definition}\textbf{\textsc{Coalitional Manipulation (CM)}}\\
 Given a voting rule $r$, a set of complete votes $\mathcal{P}$ corresponding to the votes of the non-manipulators, a non-empty set of manipulators $M$, and a candidate $c$, does there exist a way to cast the manipulators' votes such that $c$ wins the election.
\end{definition}

\begin{definition}\textbf{\textsc{Possible Winner (PW)}}\\
 Given a voting rule $r$, a set of partial votes $\mathcal{P}$, and a candidate $c$, does there exist an extension of the partial votes in $\mathcal{P}$ to linear votes that makes the candidate $c$ win the election.
\end{definition}

We now define two computational problems, namely \textsc{Weak Manipulation} and \textsc{Strong Manipulation}, below.

\begin{definition}\textbf{\textsc{Weak Manipulation (WM)}}\\
 Given a voting rule $r$, a set of partial votes $\mathcal{P}$ corresponding to the votes of the non-manipulators, a non-empty set of manipulators $M$, and a candidate $c$, does there exist a way to cast the manipulators' votes such that $c$ wins the election in least one extension of the partial votes in $\mathcal{P}$.
\end{definition}

{\bf Note:} We remark that in any instance of the WM problem, there must be at least one empty vote corresponding to the votes of the manipulators (and since there is at least one manipulator). For this reason, it is not at all clear whether the PW problem reduces to the WM problem or not.

\begin{definition}\textbf{\textsc{Strong Manipulation (SM)}}\\
 Given a voting rule $r$, a set of partial votes $\mathcal{P}$ corresponding to the votes of the non-manipulators, a non-empty set of manipulators $M$, and a candidate $c$, does there exist a way to cast the manipulators' votes such that $c$ always wins the election in every extension of the partial votes in $\mathcal{P}$.
\end{definition}

The WM and the SM problems are generalization of the CM problem. Hence, we have the following proposition.
\begin{proposition}\label{cm_hard}
 The CM problem many to one reduces\footnote{We refer to~\cite{garey1979computers} for many to one reductions.} to both the WM and the SM problems for every voting rule.
\end{proposition}

Also, the WM problem is a specialization of the PW problem where there exist at least one empty vote. This leads us to the following proposition.
\begin{proposition}\label{pw_hard}
 The WM problem many to one reduces to the PW problem for every voting rule.
\end{proposition}

\MakeUppercase proposition\nobreakspace \ref {pw_hard} shows that whenever the PW problem is in \Pb for a voting rule, the WM problem is also in \Pb for that voting rule.

\section{Hardness Results}\label{sec:hard}

In this section, we present the intractability results for the WM and the SM problems for various voting rules.
\begin{theorem}\label{wm_k_hard}
The WM problem even with one manipulator is \NPC{} for the $k$-approval voting rule, for any constant $k>1$.
\end{theorem}
\begin{proof}
For simplicity of presentation, we prove the theorem for $2$-approval.
We reduce from the PW problem for $2$-approval which is \NPC{}~\cite{betzler2009towards}. 
Let $\mathcal{P}$ be the set of partial votes in a PW instance, and 
let ${\mathcal C} := \{c_1, \ldots, c_m, c\}$ be the set of candidates, 
where the goal is to check if there is an extension of $\mathcal{P}$ that makes $c$ win. For developing the instance of WM, we need to ``reverse'' any advantage that the candidate $c$ obtains from the vote of the manipulator. Notice that the most that the manipulator can do is to increase the score of $c$ by one. Therefore, in our construction, we \textit{``artificially''} increase the score of all the other candidates by one, so that despite of the manipulator's vote, $c$ will win the new election if and only if it was a possible winner in the PW instance. To this end, we introduce $(m+1)$ many \textit{dummy} candidates $d_1, \ldots, d_{m+1}$ and the complete votes:
$$w_i := c_i \succ d_i \succ \cdots,\text{ for every } i\in \{1, \dots, m\}$$
Further, we extend the given partial votes of the PW instance to force the dummy candidates 
to be preferred least over the rest - by defining, for every $v_i \in \mathcal{P}$, the corresponding partial vote $v_i^{\prime}$ as follows.
$$v_i^\prime := v_i \cup \{{\mathcal C} \succ \{d_1, \ldots, d_{m+1}\}\}.$$ 
This ensures that all the dummy candidates do not receive any score 
from the modified partial votes corresponding to the partial votes of the PW instance. Let $({\mathcal C^\prime},\mathcal{Q},c)$ denote this constructed WM instance. We claim that the two instances are equivalent.

In the forward direction, suppose $c$ is a possible winner with respect to $\mathcal{P}$, and let $\overline{\mathcal{P}}$ be an extension where $c$ wins. Then it is easy to see that the manipulator can make $c$ win in some extension by placing $c$ and $d_{m+1}$ in the first two positions (note that the partial score of $d_{m+1}$ is zero in $\mathcal{Q}$). Indeed, consider the extension of $\mathcal{Q}$ obtained by mimicking the extension $\overline{\mathcal{P}}$ on the ``common'' partial votes, $\{v_i^\prime ~|~ v_i \in \mathcal{P}\}$. Notice that this is well-defined since $v_i$ and $v_i^\prime$ have exactly the same set of incomparable pairs. In this extension, the score of $c$ is easily seen to be strictly greater than the scores of all the other candidates, since the scores of all candidates in $\mathcal{C}$ is exactly one more than their scores in $\mathcal{P}$, and all the dummy candidates have a score of at most one. 

In the reverse direction, notice that the manipulator puts the candidates $c$ and $d_{m+1}$ in the top two positions without loss of generality. Now, suppose the manipulator's vote $c\succ d_{m+1}\succ \cdots$ makes $c$ win the election for an extension $\overline{\mathcal{Q}}$ of $\mathcal{Q}$. Then consider the $\overline{\mathcal{P}}$ obtained by restricting $\overline{\mathcal{Q}}$ to $\mathcal{C}$. Notice that the scores of all the candidates in $\mathcal{C}$ in this extension is one less than their scores in $\mathcal{Q}$. Therefore, the candidate $c$ wins this election as well, concluding the proof. 

The above proof can be imitated for any other constant values of $k$ by reducing it from the PW problem for $k$-approval and introducing $(m+1)(k-1)$ dummy candidates.
\end{proof}

The following Lemma has been used before (Lemma 4.2~\cite{baumeister2011computational}). We will use this in the proof of \MakeUppercase theorem\nobreakspace \ref {k_veto_wm}.
\begin{lemma}\label{score_gen}
Let $\mathcal{C} = \{c_1, \ldots, c_m\} \cup \{d\}$ be a set of candidates. Then, for any given $\mathbf{X} = (X_1, \cdots, X_m) \in \mathbb{Z}^m$, and for any $k < m$, there exists $\lambda\in \mathbb{N}$ and a voting profile such that the $k$-approval score of $c_i$ is $\lambda + X_i$ for all $1\le i\le m$,  and the score of $d$ is less than $\lambda + X_i$ for all $1\le i\le m$. Moreover, the number of votes is $O\left(m\sum_{i=1}^m |X_i|\right)$, where $|X_i|$ denotes the absolute value of $X_i$.
\end{lemma}

Note that, the number of votes used in \MakeUppercase lemma\nobreakspace \ref {score_gen} is a polynomial in $m$ if every $|X_i|, i\in [m]$ is a polynomial in $m$, which indeed is the case in all the proofs that use \MakeUppercase lemma\nobreakspace \ref {score_gen}.

\begin{theorem}\label{k_veto_wm}
 The WM problem for the $k$-veto voting rule is \NPC{} for any constant $k>1$.
\end{theorem}

\begin{proof}
We reduce from the PW problem for the $k$-veto voting rule. Let $\mathcal{P}$ be the set of partial votes in a PW problem instance, and 
let ${\mathcal C} := \{c_1, \ldots, c_m, c\}$ be the set of candidates, 
where the goal is to check if there is an extension that makes $c$ win with respect to $k$-veto. We assume without loss of generality that $c$'s position is fixed in all the partial votes (if not, then we fix the position of $c$ as high as possible).

We introduce $k+1$ many \textit{dummy} candidates $d_1, \ldots, d_k, d$. The role of the first $k$ dummy candidates is to ensure that the manipulator is forced to place them at the ``bottom $k$'' positions of her vote, so that all the original candidates get the same score from the additional vote of the manipulator. The most natural way of achieving this is to ensure that the dummy candidates have the same score as $c$ in any extension (note that, we know the score of $c$ since $c$'s position is fixed in all the partial votes). This would force the manipulator to place these $k$ candidates in the last $k$ positions. Indeed, doing anything else will cause these candidates to tie with $c$, even when there is an extension of $\mathcal{P}$ that makes $c$ win.

To this end, we begin by placing the dummy candidates in the top $k$ positions in all the partial votes. Formally, we modify every partial vote as follows:
$$w := d_i \succ \cdots,\text{ for every } i\in \{1, \dots, k\}$$
At this point, we know the scores of $c$ and $d_i,\text{ for every } i\in \{1, \dots, k\}$. Using \MakeUppercase lemma\nobreakspace \ref {score_gen}, we add complete votes such that the final score of $c$ is same with the score of every $d_i$ and the score of $c$ is strictly more than the score of $d$. The relative score of every other candidate remains the same.

However, this might go too far --- indeed, note that the dummy candidates may enjoy a score that is now strictly better than the score of $c$ in any extension (especially when $c$ is forced in the bottom $k$ zone of a partial vote, while a dummy candidate might make its way to the top $m-k$). So we let $t_i$ denote the difference in the score of $c$ and the score $d_i$ when they are placed in the best possible way in all the partial votes. We now use the block gadget along with the extra dummy candidate $d$ to decrease the score of $d_i$ by $t_i$. At this point, the dummy candidate $d$ has a very high score (say $t$), so we add sufficient circular votes to ensure that the score of $d$ decreases by $t$, while the relative scores of all the other candidates stay the same. 

This completes the description of the construction, and we denote the augmented set of partial votes by $\mathcal{P}^\prime$. We now argue the correctness. In the forward direction, if there is an extension of the votes that makes $c$ win; then we repeat this extension, and the vote of the manipulator puts the candidate $d_i$ at the position $m+i+2$; and all the other candidates in an arbitrary fashion. Formally, we let the manipulator's vote be:
$$\mathfrak{m} := c \succ c_1 \succ \cdots \succ c_m \succ d \succ d_1 \succ \cdots d_k.$$
It is easy to see that $c$ wins the election in this particular setup. In the reverse direction, consider a vote of the manipulator and an extension $\mathcal{Q}^\prime$ of $\mathcal{P}^\prime$ in which $c$ wins. Note that the manipulator's vote necessarily places the candidates $d_i$ in the bottom $k$ positions --- indeed, if not, then $c$ cannot win the election by construction. We extend a partial vote $w \in \mathcal{P}$ by mimicking the extension of the corresponding partial vote $w^\prime \in \mathcal{P}^\prime$, that is, we simply project the extension of $w^\prime$ on the original set of candidates $\mathcal{C}$. Let $\mathcal{Q}$ denote this proposed extension. We claim that $c$ wins the election given by $\mathcal{Q}$. Indeed, suppose not. Let $c_i$ be a candidate whose score is at least the score of $c$ in the extension $\mathcal{Q}$. Note that the scores of $c_i$ and $c$ in the extension  $\mathcal{Q}^\prime$ are exactly the same as their scores in $\mathcal{Q}$, except for a constant offset --- importantly, their scores are offset by the same amount. This implies that the score of $c_i$ is at least the score of $c$ in $\mathcal{Q}^\prime$ as well, which is a contradiction. Hence, the two instances are equivalent.
\end{proof}

Next, we prove that the WM problem for the maximin voting rule is \NPC{}. To this end, we reduce from the Exact Cover by 3-Sets (X3C) problem, which is known to be \NPC{}~\cite{garey1979computers}. The X3C problem is defined as follows.

\begin{definition}(X3C Problem)\\
 Given a set $\mathcal{U}$ and a family $\mathcal{S} = \{S_1,S_2, \dots, S_t\}$ of $t$ subsets of $\mathcal{U}$ with $|S_i|=3, 
 \forall i=1, \dots, t,$ does there exist an index set $I\subseteq \{1,\dots,t\}$ 
 with $|I|=\frac{|\mathcal{U}|}{3}$ such that $\cup_{i\in I} S_i = \mathcal{U}$.
\end{definition}

The following Lemma has been used before~\cite{mcgarvey1953theorem,xia2008determining}.

\begin{lemma}\label{maxminexist}
 For any function $f:\mathcal{C} \times \mathcal{C} \longrightarrow \mathbb{Z}$, such that
 \begin{enumerate}
  \item $\forall a,b \in \mathcal{C}, f(a,b) = -f(b,a)$.
  \item $\forall a,b \in \mathcal{C}, f(a,b)$ is even,
 \end{enumerate}
 there exists a $n$ voters' profile such that for all $a,b \in \mathcal{C}$, $a$ defeats 
 $b$ with a margin of $f(a,b)$. Moreover, 
 $$n = O\left(\sum_{\{a,b\}\in \mathcal{C}\times\mathcal{C}} |f(a,b)|\right)$$
\end{lemma}

\begin{theorem}\label{wm_maximin_hard}
The WM problem is \NPC{} for the maximin voting rule when we have only one manipulator. 
\end{theorem}

\begin{proof}
 We reduce X3C to WM for the maximin voting rule. Let $(\mathcal{U},\mathcal{S})$ be an X3C instance and $|\mathcal{U}|=m$. Assume $t=|\mathcal{S}|$ is even without loss of generality (if not, replicate any set from $\mathcal{S}$ ). We construct a corresponding WM instance for the maximin voting rule as follows.
 $$ \text{Candidate set } \mathcal{C} = \mathcal{U} \cup \{c, w, w_1, w_2, w_3\} $$
 Partial votes $\mathcal{P}$ :
\begin{equation*}
\forall i\le t, (\mathcal{U}\setminus S_i) \succ w\succ S_i\succ c\succ w_1\succ w_2\succ w_3 \setminus \left( \{w\} \times S_i\cup \{c\} \right)
\end{equation*}
 Now we add $poly(m,t)$ many complete votes such that, when combined with the profile $\mathcal{P}$, we achieve the following margin of victories in pairwise elections: 
 \begin{itemize}
  \item $D(c,w) = -2t+\frac{2q}{3}$
  \item $D(c,w_1) = -t$
  \item $D(w,w_1) = -4t$
  \item $D(w_1,w_2) = D(w_2,w_3) = D(w_3,w_1) =-t-2$
  \item $D(u_i,w) = -2t+2, \forall u_i\in \mathcal{U} $
 \end{itemize}
 For all candidate pairs that are not specified, we let $D(x,y) \in \{-1,0,1\}$. Note that this is possible due to \MakeUppercase lemma\nobreakspace \ref {maxminexist}.
 We have only one manipulator who tries to make the candidate $c$ win. This completes the description of the WM instance. We now turn to the proof of equivalence. 
 
 In the forward direction, let $S_1, \dots, S_{\frac{q}{3}}$ be an exact set cover. Then consider the following vote of the manipulator:
 $$ c\succ w\succ w_1\succ w_2\succ w_3\succ \mathcal{U}. $$
It can be easily checked that the vote above, combined with the following extension of $\mathcal{P}$ makes $c$ a maximin winner. 
 $$ i\le \frac{q}{3}, (\mathcal{U}\setminus S_i) \succ w\succ S_i\succ c\succ w_1\succ w_2\succ w_3 $$
 $$ i\ge \frac{q}{3} + 1, (\mathcal{U}\setminus S_i) \succ S_i\succ c\succ w\succ w_1\succ w_2\succ w_3 $$
 
Indeed, notice that $c$ is ahead of $w$ in $(t - 2q/3)$ votes, so in the resulting election, $D(c,w)$ improves to $(-2t + 2q/3 + t - 2q/3 = -t)$. Further, the scores $D(u_i,w)$ increase by at most $(t-1)$, since there is at least one $j \leq q/3$ (corresponding to the set that covers the element $u_i$) where $w$ is ahead of $u_i$. The scores between $w_1, w_2$ and $w_3$ increase by at most one, and therefore their final scores are (at most) $(-t-1)$. It can be easily checked, therefore, that $c$ is the unique maximin winner in this profile.
 
For the other direction, notice that irrespective of how the manipulator is voting, one of $w_1, w_2$ 
 and $w_3$ will have maximin score of $-t-1$. Hence $ D(c,w) $ has to be at least $-t$ from all the votes except manipulator's vote. This means that $c \succ w$ in at least $t-\frac{q}{3}$ extensions in $\mathcal{P}$. We claim that the sets corresponding to the rest of the $\frac{q}{3}$ extensions must form an exact set cover. If not, then some $u_i$ is uncovered and $D(u_i,w) = -t+2 $ which contradicts the fact that $c$ is the unique winner.
\end{proof}
We now prove that the WM problem for the Copeland rule is \NPC{}.

\begin{theorem}\label{wm_copeland_hard}
The WM problem is \NPC{} for Copeland voting rule when we have only one manipulator. 
\end{theorem}

\begin{proof}
 We reduce X3C to WM for Copeland rule. Let $(\mathcal{U},\mathcal{S})$ is an X3C instance. Assume $t=|\mathcal{S}|$ is odd without loss of generality (if not, replicate any set from $\mathcal{S}$ ). Let $m=|\mathcal{U}|$. We construct a corresponding WM instance for Copeland as follows.
 $$ \text{Candidate set } \mathcal{C} = \mathcal{U} \cup \{c, w, z, d\} $$
 Partial votes $\mathcal{P}$ :
 $$ \forall i\le t, (\mathcal{U}\setminus S_i) \succ z\succ c\succ d\succ S_i\succ w \setminus \{ \{z,c\} \times S_i\cup \{d,w\} \}$$
 Now we add $poly(m,t)$ many complete votes to achieve the following margin of victories in pairwise elections.
 \begin{itemize}
  \item $D(c,d) = D(c,z) = D(z,d) = D(w,c) = 4t$
  \item $D(u_i,d) = D(z,u_i) =  = 4t, \forall u_i\in \mathcal{U}$
  \item $D(c,u_i) = t-1, \forall u_i\in \mathcal{U}$
  \item $D(z,w) = t-\frac{2q}{3}+1$
 \end{itemize}
 We also arrange the score differences among $u_i,u_j$ arbitrarily, while ensuring that every $u_i$ is defeated by at least one $u_j$.
 We have only one manipulator who tries to make $c$ winner. This completes the description of the WM instance. We now turn to the proof of equivalence. 
 
 In the forward direction, suppose (by renaming) that $S_1, \dots, S_{\frac{q}{3}}$ forms an exact set cover. We claim that the following extensions, then, makes $c$ a Copeland winner.
 $$ \text{Manipulator's vote : } c\succ w\succ z\succ d\succ \mathcal{U} $$
 Extension of $\mathcal{P}$ :
 $$ i\le \frac{q}{3}, (\mathcal{U}\setminus S_i) \succ z\succ c\succ d\succ S_i\succ w $$
 $$ i\ge \frac{q}{3} + 1, (\mathcal{U}\setminus S_i) \succ d\succ S_i\succ w\succ z\succ c $$
 
 Indeed, notice that for any candidate $u_i$, $u_i \succ c$ in at most $(t-1)$ votes, since there is at least one vote (corresponding to the set that covers the element $u_i$) where $c \succ u_i$. Together with the vote of the manipulator, we have that $D(c,u_i) = 1$. Therefore, $c$ defeats all the candidates in $\mathcal{U}$, and since $c$ already defeats the candidates $d$ and $z$, it's Copeland score is $(m+2)$. It is easy to check that the best Copeland score achieved by any other candidate is at most $(m+1)$: for instance, the candidate $z$ defeats all candidates in $\mathcal{U}$ and $d$, but is defeated by $c$ and $w$ in the proposed extension. Indeed, $w \succ z$ in $(t - 2q/3 + 1)$ votes (these are the votes corresponding to $i > q/3$ and the vote of the manipulator), so $D(z,w) = 0$ in the extended profile. All other candidates are easily seen to be defeated by at least three other candidates, and therefore their score is not more than $(m+1)$.  
 
 For the reverse direction, notice that Copeland score of $z$ is at least $m+1$ since it defeats all candidates in $\mathcal{U}$ and $d$. Also notice that the Copeland score of $c$ can be at most $m+2$ since $c$ loses to $w$. Hence the only way $c$ can be the unique winner is that $c$ defeats all candidates in $\mathcal{U}$ and $z$ does not defeat $w$. Without loss of generality, we assume that $c$ is at top place in the manipulator's vote and $w\succ z$ in the manipulator's vote. This require $w\succ z$ in at least $t-\frac{q}{3}$ extensions of $\mathcal{P}$. The sets $S_i$ in the remaining of the extensions where $z\succ w$ forms an exact set cover. If not then suppose, some candidate $u_i\in \mathcal{U}$ is not covered. Then, notice that $u_i \succ c$ in all $t$ votes, making $D(c,u_i) = -1$. Even with the manipulator voting in favor of $c$, we have that can not defeat $u_i$ and thus can not achieve the desired Copeland score of $m+2$.
\end{proof}

We have following result for the SM problem for the Copeland rule.

\begin{theorem}\label{sm_copeland_hard}
SM is \NPH{} for Copeland voting rule when we have only one manipulator. 
\end{theorem}

\begin{proof}
 Change following in the reduction in the proof of \MakeUppercase theorem\nobreakspace \ref {wm_copeland_hard}. Here we assume that $t$ is even.
 \begin{itemize}
  \item $ D(c,d) = -\infty $
  \item $ D(c,u_i) = t $
  \item $ D(z,w) = t-\frac{2q}{3}-2 $
 \end{itemize}
 We have only one manipulator who tries to make $z$ winner. The proof is similar to the proof of \MakeUppercase theorem\nobreakspace \ref {wm_copeland_hard} with the only difference being here we can assume without loss of generality that manipulator puts $z$ at first position and $c$ at last position.
\end{proof}

\begin{theorem}\label{wm_bucklin_hard}
The WM problem is \NPC{} for Bucklin voting rule when we have only one manipulator. 
\end{theorem}

\begin{proof}
We reduce the X3C problem to WM for Bucklin. The reduction is along the lines of the reduction given in \cite{xia2008determining}. Let $\mathcal{U} = \{u_1, \ldots, u_m\}$ and $\mathcal{S} := \{S_1,S_2, \dots, S_t\}$ be an instance of X3C, where each $S_i$ is a subset of $\mathcal{U}$ of size three. We construct a WM instance based on $(\mathcal{U},\mathcal{S})$ as follows.
$$ \text{Candidate set : } \mathcal{C} = W\cup D\cup \mathcal{U}\cup \{c,w\},$$ 
$$\text{ where } W=\{w_1,\ldots, w_{m+1}\}, D=\{d_1,\ldots, d_{m+1}\}$$
We first introduce the following partial votes $P_1$ in correspondence with the sets in the family as follows.
$$ w_1\succ \ldots \succ w_{m+1}\succ S_i\succ c\succ (\mathcal{U}\setminus S_i)\succ D$$ 
$$ \setminus \left(\{\{w_{m-2},w_{m-1},w_q,w_{m+1}\} \times (\{c\}\cup S_i)\}\right), \forall i\le t$$
Further, we introduce the following additional complete votes $P_2$:
\begin{itemize}
 \item $t$ copies of $\mathcal{U}\succ c\succ \ldots$
 \item $\frac{m}{3}-1$ copies of $\mathcal{U}\succ w\succ c\succ \ldots$
 \item $\frac{m}{3}+1$ copies of $D\succ w_1\succ \ldots$
\end{itemize}
The total number of voters, including the manipulator, is $2t+\frac{2m}{3}+1$. Now we show equivalence of the two instances. 

In the forward direction, suppose we have an exact set cover. Let the vote of the manipulator be $c\succ D\succ \cdots$. This makes $c$ win the election for the extension where the votes in $P_1$ are extended as follows.

If $S_i$ is part of set cover then we have:
\begin{equation*}
w_1\succ \ldots \succ w_{m-3}\succ S_i\succ c\succ w_{m-2}\succ \ldots \succ w_{m+1} \succ (\mathcal{U}\setminus S_i)\succ D
\end{equation*}
On the other hand, if $S_i$ is not part of set cover then we have:
$$w_1\succ \ldots \succ w_{m+1}\succ S_i\succ c\succ (\mathcal{U}\setminus S_i)\succ D.$$

Notice that the Bucklin score of $c$ is $m+1$ in this extension, since it appears in the top $m+1$ positions in the $m/3$ votes corresponding to the set cover, $t$ votes from the complete profile $P_2$ and one vote of the manipulator. 
For any other candidate $u_i\in \mathcal{U}$, $u_i$ appears in the top $m+1$ positions once in $P_1$ and $t+\frac{m}{3}-1$ times in $P_2$. Thus, $u_i$ does not get majority in top $m+1$ positions making its Bucklin score at least $m+2$.
For any other candidate $s_i$ in $S$, $s_i$ appears at most $(t - m/3 + 1)$ times in the top $(m+2)$ positions among the votes in $P_1$ (since the $m/3$ extensions corresponding to the set cover admit at most one instance of $s_i$ in the top $m+2$ spots). Further, note that $s_i$ appears exactly $(t + m/3 - 1)$ times in the top $(m+2)$ positions in $P_2$. Therefore, the total number of appearances of any $s_i \in S$ among the top $(m+2)$ positions, across the voting profile, is at most $(2t + 2m/3)$, which is shy of majority. 
Similarly, the candidate $w_1$ appears in the top $(m+2)$ positions $(2t + 2m/3 + 1)$ times, but only $t$ times in the top $(m+1)$. 
Similarly, the candidate $w_1$ appears only $t$ times in the top $m+1$ positions. 
The same can be argued for the remaining candidates in $D, W$ and $w$. 

In the reverse direction, suppose the WM is a yes instance. We may assume without loss of generality that the manipulator's vote is $c\succ D\succ \cdots$, since the Bucklin score of the candidates in $D$ is at least $2m$. As $w_1$ is ranked within top $m+2$ positions in $t+\frac{m}{3}+1$ votes, for $c$ to win, $c\succ w_{m-2}$ must hold in at least $\frac{m}{3}$ votes in $P_1$. In those votes, all the candidates in $S_i$ are also within top $m+2$ positions. Now if any candidate in $\mathcal{U}$ is within the top $m+2$ positions in $P_1$ more than once, then $c$ will not be the unique winner. Hence, the $S_i$'s corresponding to the votes where $c\succ w_{m-2}$ form an exact set cover. 
\end{proof}

\section{Polynomial Time Algorithms for the WM and SM Problems}\label{sec:poly}

Now we present polynomial time algorithms for the WM and SM problems for various voting rules.

\begin{theorem}\label{pv_easy}
 The WM problem is in \Pb{} for the plurality and veto voting rules.
\end{theorem}

\begin{proof}
 The PW problem is in \Pb{} for the plurality and the veto voting rules~\cite{betzler2009towards}. Hence, the result follows from \MakeUppercase proposition\nobreakspace \ref {pw_hard}.
\end{proof}

\begin{theorem}\label{sm_k_easy}
The SM problem is in \Pb{} for the $k$-approval and $k$-veto voting rules, for any $k$.
\end{theorem}

\begin{proof}
 For the time being, we just concentrate on non-manipulators' votes. For each candidate $c^{\prime} \in \mathcal{C}\setminus\{c\}$, calculate the maximum possible value of $s^{max}_{NM}(c,c^{\prime}) = s_{NM}(c^{\prime}) - s_{NM}(c)$ from non-manipulators' votes, where $s_{NM}(c)$ $(s_{NM}(c^{\prime}))$ is the score that the candidate $c$ $(\text{respectively }c^{\prime})$ receives from the votes of the non-manipulators. This can be done by checking all possible $O(m^2)$ pair of positions for $c$ and $c^{\prime}$ in each vote $v$ and choosing the one which maximizes $s_v(c^{\prime}) - s_v(c)$ from that vote. Now, we fix the position of $c$ at the top position for the manipulators' votes and we check if it is possible to place other candidates in the manipulators' votes such that the final value of $s^{max}_{NM}(c,c^{\prime}) + s_M(c^{\prime}) - s_M(c)$ is negative which can be solved easily by reducing it to the max flow problem.
\end{proof}

We note that \MakeUppercase theorem\nobreakspace \ref {sm_k_easy} and \MakeUppercase theorem\nobreakspace \ref {pv_easy} hold for any number of manipulators. We now prove that the SM problem for scoring rules is in \Pb{} for one manipulator.

\begin{theorem}\label{sm_sr_easy}
The SM problem is in \Pb{} for any scoring rule when we have only one manipulator.
\end{theorem}

\begin{proof}
For each candidate $c^{\prime} \in \mathcal{C}\setminus\{c\}$, calculate $s^{max}_{NM}(c,c^{\prime})$ using same technique described in the proof of \MakeUppercase theorem\nobreakspace \ref {sm_k_easy}. Now, we put $c$ at the top position of the manipulator's vote. For each candidate $c^{\prime} \in \mathcal{C}\setminus\{c\}$, $c^{\prime}$ can be placed at positions $i\in \{2,\ldots,m\}$ in the manipulator's vote which makes $s^{max}_{NM}(c,c^{\prime}) + \alpha_i - \alpha_1$ negative. Using this, construct 
a bipartite graph with $\mathcal{C}\setminus\{c\}$ on left and $\{2, \dots, m\}$ 
on right and there is an edge between $c^{\prime}$ and $i$ iff the candidate $c^{\prime}$ 
can be placed at $i$ in the manipulator's vote according to the above criteria. 
Now solve the problem by finding existence of perfect matching in this graph.
\end{proof}

Next, we move on to the SM problem for the Bucklin rule.

\begin{theorem}\label{sm_bucklin_easy}
The SM problem is in \Pb{} for the Bucklin voting rule for any number of manipulators. 
\end{theorem}

\begin{proof}
Let $(\mathcal{C},\mathcal{P},M,c)$ be an instance of SM for Bucklin, and let $m$ denote the total number of candidates in this instance. Recall that the manipulators have to cast their votes so as to ensure that the candidate $c$ wins in every possible extension of $\mathcal{P}$. We use $\mathcal{Q}$ to denote the set of manipulating votes that we will construct. To begin with, without loss of generality, the manipulators place $c$ in the top position of all their votes. Now, we have to organize the positioning of the remaining candidates across the votes of the manipulators to ensure that $c$ is a necessary winner of the election $\{\mathcal{C},\mathcal{P}\cup \mathcal{Q}\}$.

To this end, we would like to develop a system of constraints indicating the overall number of times that we are free to place a candidate $w \in \mathcal{C} \setminus \{c\}$ among the top $\ell$ positions in the profile $\mathcal{Q}$. In particular, let us fix $w \in \mathcal{C} \setminus \{c\}$ and $2 \leq \ell \leq m$. Let $\eta_{w,\ell}$ be the maximum number of votes of $\mathcal{Q}$ in which $w$ can appear in the top $\ell$ positions. Our first step is to compute necessary conditions for $\eta_{w,\ell}$.

We use $\overline{\mathcal{P}}_{w,\ell}$ to denote a set of complete votes that we will construct based on the given partial votes. Intuitively, these votes will represent the ``worst'' possible extensions from the point of view of $c$ when pitted against $w$. These votes are engineered to ensure that the manipulators can make $c$ win the elections $\overline{\mathcal{P}}_{w,\ell}$ for all $w \in \mathcal{C} \setminus \{c\}$ and $\ell \in \{2,\ldots,m\}$, if, and only if, they can strongly manipulate in favor of $c$. More formally, there exists a voting profile $\mathcal{Q}$ of the manipulators so that $c$ wins the election $\overline{\mathcal{P}}_{w,\ell} \cup \mathcal{Q}$, for all $w \in \mathcal{C} \setminus \{c\}$ and $\ell \in \{2,\ldots,m\}$, if, and only if, $c$ wins in every extension of the profile $\mathcal{P} \cup \mathcal{Q}$. 

We now describe the profile $\overline{\mathcal{P}}_{w,\ell}$. The construction is based on the following case analysis, where our goal is to ensure that, to the extent possible, we position $c$ out of the top $\ell-1$ positions, and incorporate $w$ among the top $\ell$ positions. 

\begin{itemize}
 \item Let $v \in \mathcal{P}$ be such that either $c$ and $w$ are incomparable or $w \succ c$. We add the complete vote $v^\prime$ to  $\overline{\mathcal{P}}_{w,\ell}$, where $v^\prime$ is obtained from $v$ by placing $w$ at the highest possible position and $c$ at the lowest possible position, and extending the remaining vote arbitrarily. 
 \item Let $v \in \mathcal{P}$ be such that $c \succ w$, but there are at least $\ell$ candidates that are preferred over $w$ in $v$. We add the complete vote $v^\prime$ to  $\overline{\mathcal{P}}_{w,\ell}$, where $v^\prime$ is obtained from $v$ by placing $c$ at the lowest possible position, and extending the remaining vote arbitrarily.  
 \item Let $v \in \mathcal{P}$ be such that $c$ is forced to be within top $\ell-1$ positions, then we add the complete vote $v^\prime$ to  $\overline{\mathcal{P}}_{w,\ell}$, where $v^\prime$ is obtained from $v$ by first placing $w$ at the highest possible position followed by placing $c$ at the lowest possible position, and extending the remaining vote arbitrarily.
 \item In the remaining votes, notice that whenever $w$ is in the top $\ell$ positions, $c$ is also in the top $\ell-1$ positions. Let $\mathcal{P}^*_{w,\ell}$ denote this set of votes, and let $t$ be the number of votes in $\mathcal{P}^*_{w,\ell}$.
\end{itemize}

We now consider two cases. Let $d_\ell(c)$ be the number of times $c$ is placed in the top $\ell-1$ positions in the profile $\overline{\mathcal{P}}_{w,\ell} \cup \mathcal{Q}$, and let $d_\ell(w)$ be the number of times $w$ is placed in the top $\ell$ positions in the profile $\overline{\mathcal{P}}_{w,\ell}$. Let us now formulate the requirement that in $\overline{\mathcal{P}}_{w,l} \cup \mathcal{Q}$, the candidate $c$ does \emph{not} have a majority in the top $\ell-1$ positions and $w$ \emph{does} have a majority in the top $\ell$ positions. Note that if this requirement holds for any $w$ and $\ell$, then strong manipulation is not possible. Therefore, to strongly manipulate in favor of $c$, we must ensure that for every choice of $w$ and $\ell$, we are able to negate the conditions that we derive. 

The first condition from above simply translates to $d_\ell(c) \le n/2$. The second condition amounts to requiring first, that there are at least $n/2$ votes where $w$ appears in the top $\ell$ positions, that is, $d_\ell(w) + \eta_{w,\ell} + t > n/2$. Further, note that the gap between $d_\ell(w)+\eta_{w,\ell}$ and majority will be filled by using votes from $\mathcal{P}^*_{w,\ell}$ to ``push'' $w$ forward. However, these votes  contribute equally to $w$ and $c$ being in the top $\ell$ and $\ell-1$ positions, respectively. Therefore, the difference between $d_\ell(w)+\eta_{w,\ell}$ and $n/2$ must be less than the difference between $d_\ell(c)$ and $n/2$. Summarizing, the following conditions, which we collectively denote by $(\star)$, are sufficient to defeat $c$ in some extension: $d_\ell(c) \le n/2, d_\ell(w) + \eta_{w,\ell} + t > n/2, n/2 - d_\ell(w) + \eta_{w,\ell} < n/2 - d_\ell(c)$.

From the manipulator's point of view, the above provides a set of constraints to be satisfied as they place the remaining candidates across their votes. Whenever $d_\ell(c) > n/2$, the manipulators place any of the other candidates among the top $\ell$ positions freely, because $c$ already has majority. On the other hand, if $d_\ell(c) \leq n/2$, then the manipulators must respect at least one of the following constraints: $\eta_{w,\ell} \leq n/2 - t - d_\ell(w)$ and $\eta_{w,\ell} \leq d_\ell(c) - d_\ell(w)$.

Extending the votes of the manipulator while respecting these constraints (or concluding that this is impossible to do) can be achieved by a natural greedy strategy --- construct the manipulators' votes by moving positionally from left to right. For each position, consider each manipulator and populate it's vote for that position with any available candidate. We output the profile if the process terminates by completing all the votes, otherwise, we say \textsc{No}.
We now argue the proof of correctness. Suppose, on the one hand, that the algorithm returns \textsc{No}. This implies that there exists a choice of $w \in \mathcal{C} \setminus \{c\}$ and $\ell \in \{2,\ldots,m\}$ such that for any voting profile $\mathcal{Q}$ of the manipulators, the conditions in $(\star)$ are satisfied. (Indeed, if there exists a voting profile that violated at least one of these conditions, then the greedy algorithm would have discovered it.) Therefore, no matter how the manipulators cast their vote, there exists an extension where $c$ is defeated. In particular, for the votes in $\mathcal{P} \setminus \mathcal{P}^*_{w,\ell}$, this extension is given by $\overline{\mathcal{P}}_{w,\ell}$. Further, we choose $n/2 - \eta_{w,\ell} - d_\ell(w)$ votes among the votes in $\mathcal{P}^*_{w,\ell}$ and extend them by placing $w$ in the top $\ell$ positions (and extending the rest of the profile arbitrary). We extend the remaining votes in $\mathcal{P}^*_{w,\ell}$ by positioning $w$ outside the top $\ell$ positions. Clearly, in this extension, $c$ fails to achieve majority in the top $\ell-1$ positions while $w$ does achieve majority in the top $\ell$ positions. 

On the other hand, if the algorithm returns \textsc{Yes}, then consider the voting profile of the manipulators. We claim that $c$ wins in every extension of $\mathcal{P} \cup \mathcal{Q}$. Suppose, to the contrary, that there exists an extension $\mathcal{R}$ and a candidate $w$ such that the Bucklin score of $c$ is no more than the Bucklin score of $w$ in $\mathcal{R}$. In this extension, therefore, there exists $\ell \in \{2,\ldots,m\}$ for which $w$ attains majority in the top $\ell$ positions and $c$ fails to attain majority in the top $\ell-1$ positions. However, note that this is already impossible in any extension of the profile $\overline{\mathcal{P}}_{w,l} \cup \mathcal{P}^*_{w,\ell}$, because of the design of the constraints. By construction, the number of votes in which $c$ appears in the top $\ell-1$ positions in $\mathcal{R}$ is only greater than the number of times $c$ appears in the top $\ell-1$ positions in any extension of $\overline{\mathcal{P}}_{w,l} \cup \mathcal{P}^*_{w,\ell}$ (and similarly for $w$). This leads us to the desired contradiction. 
\end{proof}

For the maximin voting rule, we have the following result.

\begin{theorem}\label{sm_maximin_easy}
The SM problem for the scoring rules and the maximin voting rules are in \Pb{}, when we have only one manipulator. The WM problem for the plurality and veto voting rules are in \Pb{}.
\end{theorem}

\begin{proof}
 For the time being, just concentrate on non-manipulators' votes. Using the algorithm for NW for maximin in \cite{xia2008determining}, we compute for all pairs $w, w^{\prime} \in \mathcal{C}$, $N_{(w,w^{\prime})}(w,d)$ and $N_{(w,w^{\prime})}(c,w^{\prime})$ for all $d\in \mathcal{C}\setminus \{c\}$. This can be computed in polynomial time. Now we place $c$ at the top position in the manipulator's vote and increase all $N_{(w,w^{\prime})}(c,w^{\prime})$ by one. Now we place a candidate $w$ at the second position if for all $w^{\prime} \in \mathcal{C}$, $N_{(w,w^{\prime})}^{\prime}(w,d) < N_{(w,w^{\prime})}(c,w^{\prime})$ for all $d\in \mathcal{C}\setminus \{c\}$, where $N_{(w,w^{\prime})}^{\prime}(w,d) = N_{(w,w^{\prime})}(w,d)$ of the candidate $d$ has already been assigned some position in the manipulator's vote, and $N_{(w,w^{\prime})}^{\prime}(w,d) = N_{(w,w^{\prime})}(w,d)+1$ else. The correctness argument is in the similar lines of the classical greedy manipulation algorithm of Bartholdi and Orlin.
\end{proof}

\section{Conclusion}\label{sec:con}
We studied manipulation under an incomplete information setting thereby filling an important research gap. The incomplete information setting is more realistic than the classical complete information setting and hence our results in this setting are significant from both theoretical and application point of view. We leave open the problem of completely establishing the complexity of strong and weak manipulation for all the scoring rules. Other fundamental forms of manipulation and control do exist in voting, such as destructive manipulation and control by adding candidates. An interesting line of future work is to explore the complexity of these problems in a partial information setting.
Although it has been shown that the complexity barrier against manipulation is weak~\cite{procaccia2006junta,faliszewski2010ai,walsh2010empirical}, we believe that studying the computational complexity of manipulation of voting rules is still an interesting and active topic of research in computational social choice~\cite{hemaspaandra2013schulze,brandt2012computational,faliszewski2014complexity,narodytska2013manipulating,obraztsova2012optimal}.
Studying the average case complexity of manipulation in the incomplete information setting, as already done in the complete information setting~\cite{procaccia2006junta,faliszewski2010ai,walsh2010empirical}, is another very interesting research direction to pursue in future.

\bibliographystyle{apalike}
\bibliography{mp}

\end{document}